\documentclass{article}

\usepackage{arxiv-preprint,times}
\usepackage{amsmath,amssymb,amsthm}
\usepackage{booktabs}
\usepackage{graphicx}
\usepackage{float}
\usepackage{placeins}
\usepackage{microtype}
\usepackage{xcolor}
\usepackage{hyperref}
\hypersetup{hidelinks,bookmarksopen=true,bookmarksnumbered=true,
  pdftitle={Short-to-Long: Detached Features for Long-Horizon Value Learning},
  pdfsubject={Reinforcement learning},
  pdfkeywords={reinforcement learning, representation learning, value learning}}
\usepackage{url}
\usepackage{tikz}
\usetikzlibrary{arrows.meta,calc,fit,positioning}

\definecolor{longblue}{HTML}{2F5D8A}
\definecolor{shortorange}{HTML}{D97706}
\definecolor{bankgreen}{HTML}{2F855A}

\newtheorem{theorem}{Theorem}
\newcommand{\sg}{\operatorname{sg}}
\newcommand{\E}{\mathbb{E}}
\newcommand{\method}{Short-to-Long}
\newcommand{\shortmethod}{S2L}
\newcommand{\hthreed}{S2L-4}
\newcommand{\sharedhead}{MDH-4}

\title{Short-to-Long: Detached Features\\for Long-Horizon Value Learning}

\author{Tong Che \quad Boris Ivanovic \quad Marco Pavone\\
NVIDIA Research}
\hypersetup{pdfauthor={Tong Che, Boris Ivanovic, Marco Pavone}}

\begin{document}
\raggedbottom

\maketitle

\begin{abstract}
Learning useful representations from continually changing value targets
remains a challenge in deep reinforcement learning.  We study the temporal
horizon of this supervision: smaller discounts make these targets less
sensitive to distant bootstrap changes, but using them for behavior changes
the control objective.  We
introduce \emph{\method} (\shortmethod), which supplies a long-horizon
behavior learner with features trained by smaller-discount source networks.
Detached connections prevent the behavior value loss from updating the
sources, allowing each source to follow its own value objective while the
behavior discount remains unchanged.  Across Atari, online continuous control, and offline continuous control,
the approach yields aggregate performance gains, improving the paired
terminal median on 40 of 57 Atari games.  Matched Atari ablations support
the roles of feature transfer, smaller source discounts, and gradient
isolation.  Fixed-trajectory measurements further show reduced drift in
source value targets with mixed discounts.
\end{abstract}

\section{Introduction}

Deep reinforcement learning relies on representations that remain useful
as an agent's predictions and behavior change.  Learning these representations
is difficult because the supervision itself is non-stationary: policy updates
change the experience distribution, while value updates change the bootstrap
predictions used as targets.  In an end-to-end value network, these targets
also shape the hidden features on which future predictions depend.  This
feedback connects representation learning to the optimization difficulties
observed in deep RL
\citep{sokar2023dormant,castanyer2025stable,pasand2026isotropic,
lelan2023bootstrapped,tang2023representation}.

Most responses change the architecture or the optimizer: preserving
gradient flow across depth \citep{castanyer2025stable}, residual critics
\citep{lee2025simba}, or auxiliary prediction targets that shape
representations \citep{jaderberg2017unreal,schwarzer2021spr}.  The temporal
structure of the value supervision itself is left untouched.  We examine
that horizon.  In an end-to-end value network, the discount factor
determines both how future rewards enter the control objective and the
targets used to train hidden features.  These two roles need not be coupled
for every source of features.

Multiple discounts have appeared in deep RL before, but always within one
set of parameters: auxiliary heads on a shared trunk
\citep{fedus2019hyperbolic}, or a decomposition of one long-horizon value
into differences between neighboring discounts
\citep{romoff2019separating,tang2021taylor}.  Every objective then updates
the same features, so the longest-horizon loss still shapes them.

Along a fixed trajectory, reducing the discount attenuates the contribution
of a distant bootstrap change to the current target
\citep{amit2020discount,kozuno2021peng}.
We make this quantitative: a change $k$ steps away is transmitted with a
factor $\gamma^k$, at $k=32$ some $21.1\times$ smaller for $\gamma=.9$ than
for $\gamma=.99$.  A shared-feature analysis further shows that the
longest-horizon loss dominates the feature update unless its gradient is
removed, which motivates isolating the sources.
Applying the same reduction to the behavior learner would change how it
values future rewards.  We instead investigate whether shorter-horizon
value objectives can supply useful features for a learner that retains
its original control objective.

\emph{\method} (\shortmethod) implements this idea with separately trained
value networks that share experience.  Source learners use smaller discounts
and pass their hidden activations to the receiving learner through
stop-gradient connections.  Each source is updated only by its own value
loss, so the receiving loss cannot redirect it toward the longer-horizon
target.  All learners remain active throughout training, and only the final
network supplies the action values or actor objective.  The behavior learner
can therefore use shorter-horizon features while continuing to optimize
its original discounted objective.  Figure~\ref{fig:architecture} shows
the four-learner instantiation used in our main online experiments.

\paragraph{Contributions.}
\begin{itemize}\itemsep1pt\parskip0pt
\item \shortmethod{}: detached feature transfer from smaller-discount source
networks into a behavior learner whose discount and control objective are
unchanged.
\item An account of how the discount governs the transmission of bootstrap
changes into truncated $\lambda$-returns, and a shared-feature model
(Theorem~\ref{thm:target-feature}) in which the longest-horizon loss
dominates the feature update unless its gradient is removed.
\item Gains in all three settings---$+.049$ median paired human-normalized
score over the 57-game Atari suite relative to scaled PQN
\citep{castanyer2025stable}, 465 of 720 matched DeepMind Control Suite
pairs, and nine offline MuJoCo datasets---with matched ablations isolating
each design choice, direct target-drift measurements, and a
parameter-matched shared-trunk control.
\end{itemize}

\section{Short-to-Long}
\label{sec:method}
\label{sec:h38d}

\hthreed{} consists of four value learners ordered from shorter to longer
horizon.  Each learner has separate parameters, normalization statistics,
optimizer state, and targets.  They share training transitions, not parameters.

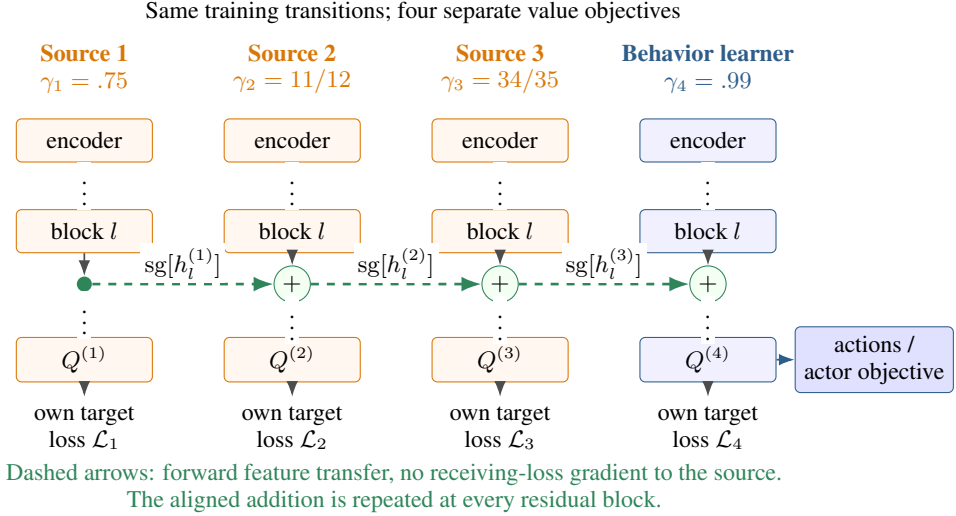
\begin{figure}[!htbp]
    \centering
    \begin{tikzpicture}[
    x=1cm,y=1cm,>=Latex,
    network/.style={draw,rounded corners=2pt,minimum width=1.8cm,
                    minimum height=.57cm,align=center,font=\small},
    source/.style={network,draw=shortorange,fill=orange!6},
    behavior/.style={network,draw=longblue,fill=blue!6},
    sum/.style={circle,draw=bankgreen,fill=green!5,inner sep=1.5pt,font=\small},
    flow/.style={->,draw=black!70,line width=.65pt},
    transfer/.style={->,draw=bankgreen,dashed,line width=1pt},
    heading/.style={align=center,font=\small},
    edge label/.style={font=\footnotesize,fill=white,inner sep=1.5pt}
]
\node[font=\small,align=center] at (5.4,3.45)
    {Same training transitions; four separate value objectives};

\foreach \k/\x/\discount/\horizon in {
    1/1.05/.75/4,
    2/3.8/{11/12}/12,
    3/6.55/{34/35}/35,
    4/9.3/.99/100
} {
    \ifnum\k=4
        \node[heading,text=longblue] at (\x,2.7)
            {\textbf{Behavior learner}\\$\gamma_{\k}=\discount$};
        \node[behavior] (enc\k) at (\x,1.75) {encoder};
        \node[behavior] (block\k) at (\x,.55) {block $l$};
    \else
        \node[heading,text=shortorange] at (\x,2.7)
            {\textbf{Source \k}\\$\gamma_{\k}=\discount$};
        \node[source] (enc\k) at (\x,1.75) {encoder};
        \node[source] (block\k) at (\x,.55) {block $l$};
    \fi
    \draw[flow] (enc\k) -- node[edge label] {$\vdots$} (block\k);
    \ifnum\k=1
        \node[circle,fill=bankgreen,inner sep=2pt] (sum\k) at (\x,-.15) {};
    \else
        \node[sum] (sum\k) at (\x,-.15) {$+$};
    \fi
    \draw[flow] (block\k) -- (sum\k);
    \ifnum\k=4
        \node[behavior] (q\k) at (\x,-1.15) {$Q^{(\k)}$};
    \else
        \node[source] (q\k) at (\x,-1.15) {$Q^{(\k)}$};
    \fi
    \draw[flow] (sum\k) -- node[edge label] {$\vdots$} (q\k);
    \node[heading] (loss\k) at (\x,-2.05) {own target\\loss $\mathcal L_{\k}$};
    \draw[flow] (q\k) -- (loss\k);
}
\foreach \a/\b in {1/2,2/3,3/4} {
    \draw[transfer] (sum\a.east) --
        node[edge label,above] {$\sg[h^{(\a)}_l]$} (sum\b.west);
}
\node[network,draw=longblue,fill=blue!12,minimum width=1.45cm,
      font=\footnotesize] (act) at (11.5,-1.15) {actions /\\actor objective};
\draw[flow,draw=longblue] (q4.east) -- (act.west);
\node[font=\footnotesize,text=bankgreen,align=center] at (5.15,-2.85)
    {Dashed arrows: forward feature transfer, no receiving-loss gradient to the source.\\
     The aligned addition is repeated at every residual block.};
\end{tikzpicture}
    \caption{\textbf{S2L-4: learn separately, transfer features from short to long.}
    All four learners train on the same transitions.  At each residual depth,
    a learner receives its predecessor's activation through a detached
    addition.  Each value loss updates only its own network.  The final
    $\gamma=.99$ learner supplies action values or the actor objective.}
    \label{fig:architecture}
\end{figure}

\subsection{Four horizons, one behavior objective}

We use the fixed discounts
\begin{equation}
 (\gamma_1,\gamma_2,\gamma_3,\gamma_4)
 =\left(.75,\frac{11}{12},\frac{34}{35},.99\right),
 \qquad
 \frac{1}{1-\gamma_k}\in\{4,12,35,100\}.
 \label{eq:h38d-gammas}
\end{equation}
The final learner, $Q^{(4)}$, retains the base algorithm's control objective
and Bellman-target construction.  The first three learners supply features;
their Q-values are not averaged into the behavior value or used to mix its
targets.  In Atari, actions are epsilon-greedy with respect to $Q^{(4)}$.
In continuous control, the actor is optimized using only its value.

\subsection{Detached additions}

Let $\bar h^{(k)}_l=f^{(k)}_l(h^{(k)}_{l-1})$ be learner $k$'s activation
after its ordinary block, including residual and skip connections.  Evaluate
the networks from shortest to longest horizon and add the preceding
learner's same-depth activation:
\begin{equation}
 h^{(1)}_l=\bar h^{(1)}_l,\qquad
 h^{(k)}_l=\bar h^{(k)}_l+\sg\!\left[h^{(k-1)}_l\right],
 \quad k=2,3,4.
 \label{eq:h38d-chain}
\end{equation}
Here $\sg$ is the identity in the forward pass and stops gradients in the
backward pass.  There are no learned gates or projection parameters.
Because the predecessor activation already includes its own transfer,
information can pass through the entire short-to-long chain.

If $\theta_j$ denotes learner $j$'s parameters and $\mathcal L_k$ denotes
learner $k$'s value loss, then
\begin{equation}
 \nabla_{\theta_j}\mathcal L_k=0\quad\text{for }j<k.
 \label{eq:isolation}
\end{equation}
Each network nevertheless continues to learn from its own loss.
The two-learner interfaces used in the development studies follow the same
principle; their definitions are in Appendix~\ref{app:read-interfaces}.

\subsection{Training and domain instantiations}

For online Atari, the four CNN--MultiSkip networks each form a Q($\lambda$)
return.  With $V^{(k)}(s)=\max_a Q^{(k)}(s,a)$ and terminal masks omitted
from the notation,
\begin{equation}
 G^{(k)}_t=r_t+\gamma_k\!\left[(1-\lambda)V^{(k)}(s_{t+1})
                    +\lambda G^{(k)}_{t+1}\right],
 \qquad G^{(k)}_T=V^{(k)}(s_T).
 \label{eq:qlambda}
\end{equation}
All learners regress on the action $a_t$ actually taken:
\begin{equation}
 \mathcal L_k=\E\!\left[(Q^{(k)}(s_t,a_t)-\sg[G^{(k)}_t])^2\right],
 \qquad \mathcal L=\sum_{k=1}^{4}\mathcal L_k.
 \label{eq:loss}
\end{equation}
The losses are summed without averaging and gradients are clipped separately
per learner.  We use $\lambda=.65$ and 32-step rollouts.

For online continuous control, four disjoint SimBa critics replace the
single critic and use one-step targets with their respective discounts.
The actor and interaction protocol are unchanged.  Critic updates detach
predecessor features as in Eq.~\eqref{eq:h38d-chain}; actor updates hold all
critic parameters fixed while retaining action derivatives through all four
critics.  Atari and continuous-control settings are recorded in
Appendices~\ref{app:details}, \ref{app:dmc-details},
and~\ref{app:offline-details}.

\paragraph{Continuous-control update.}
Each replay batch is processed in four steps:
\begin{enumerate}
    \item Form four one-step targets from the same transitions using the
    discounts in Eq.~\eqref{eq:h38d-gammas}, including the corresponding
    SAC entropy term when applicable.
    \item Evaluate the critics from shortest to longest horizon, adding
    the immediate predecessor's activation at each aligned residual block.
    \item Stop gradients through predecessor features, sum the four TD
    losses, and update each critic with its own optimizer state.
    \item Update the actor from the final $.99$ value.  Keep predecessor
    action derivatives live and critic parameters fixed, then apply the
    base algorithm's target-network update.
\end{enumerate}
The same transferred activations therefore support both value learning and
policy improvement: detachment isolates the critics' TD losses, while action
derivatives carry their influence into the actor update.

\section{Online Control Results}
\label{sec:experiments}

We compare \hthreed{} with scaled PQN on Atari and with SimBa on continuous
control, matching environment-interaction budgets and paired seeds.
Each S2L learner uses the baseline network architecture; total capacity
scales with the number of learners.  Value networks at this scale are
small enough that throughput is dominated by simulation rather than network
evaluation, which is why deep RL comparisons use matched interaction budgets
rather than matched compute; across the 57 games the four learners add
$12.3\pm6.2\%$ wall-clock training time.

\subsection{Atari: gains across the full 57-game suite}

\hthreed{} improves terminal performance on 40 of 57 games, with a median
paired game improvement of $+.049$ human-normalized score (HNS).
Figure~\ref{fig:atari57-gain} reports the complete per-game comparison.
Using a descriptive $\pm.01$ band, 36 games improve by more than $+.01$,
11 fall within the band, and 10 fall below $-.01$.

\paragraph{Setup.}
We compare with the scaled PQN agent of \citet{castanyer2025stable} using four
paired seeds and 2.50M environment steps per game.  Five games were used
for development; the other 52 were not.
We define $\mathrm{HNS}=(R_{\mathrm{agent}}-R_{\mathrm{random}})/
(R_{\mathrm{human}}-R_{\mathrm{random}})$.
For each run, terminal HNS averages the last ten logged values, each of
which averages the 20 most recently completed training episodes.
Adjacent windows may overlap.  We form differences within seed and report
the four-seed median for each game.

\begin{figure}[H]
    \centering
    \includegraphics[width=\linewidth]{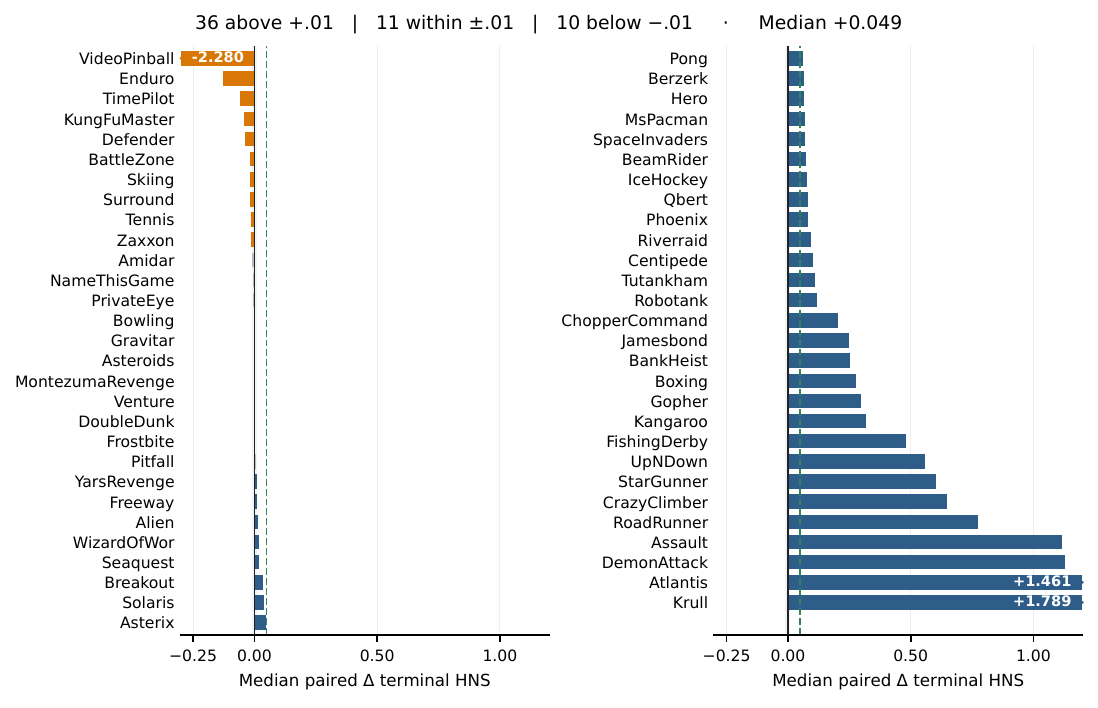}
    \caption{\textbf{S2L-4 improves the paired median on 40 of 57 Atari games.}
    Bars show median paired terminal HNS differences from scaled PQN,
    ordered from the smallest to largest effect across the two columns.
    Blue, gray, and orange indicate values above $+.01$, within $\pm.01$,
    and below $-.01$.  The dashed line marks the median across games.
    Effects outside $[-.30,+1.20]$ are annotated with their actual values.}
    \label{fig:atari57-gain}
\end{figure}

\subsection{Continuous control: larger gains late in training}
\label{sec:dmc-results}

At one million environment steps, \hthreed{} improves return in $465/720$
matched pairs, with a mean difference of $+39.09$.
The advantage emerges late: the time-averaged difference is $-12.71$ over
the first half, $+27.03$ over the second, and $+34.99$ over the final
quarter; the full-training value of $+7.16$ therefore averages two regimes
of opposite sign (Figure~\ref{fig:dmc-learning}).  Every subgroup has a
negative first-half difference but a positive endpoint mean
(Table~\ref{tab:dmc-main}).  One plausible explanation is that
source features must mature before their transfer becomes useful.

\paragraph{Setup.}
The matrix covers SAC and DDPG \citep{haarnoja2018sac,lillicrap2016ddpg},
AdamW and Kron, and four DeepMind Control Suite tasks: dog-run, dog-trot,
humanoid-run, and humanoid-walk \citep{tassa2018deepmind}.
Paired methods use the same actor architecture, individual-critic width and
depth, seed, and sample budget.  We analyze 720 full-budget pairs across
nine scales and five seeds.  Appendix~\ref{app:dmc-details} gives the coverage and scale definitions.

\paragraph{Retaining acquired return.}
\label{sec:dmc-stability}
Over 500k--1M steps, \hthreed{} has lower aggregate local fluctuation,
maximum drawdown, cumulative negative change, and worst single-evaluation
drop (Table~\ref{tab:dmc-stability}).  These return-dynamics measures quantify
retention of acquired return and are distinct from the target drift in
Section~\ref{sec:target-diagnosis}.  Appendix~\ref{app:dmc-details} gives
definitions and cross-seed variability.

\begin{figure}[!htbp]
    \centering
    \includegraphics[width=\linewidth]{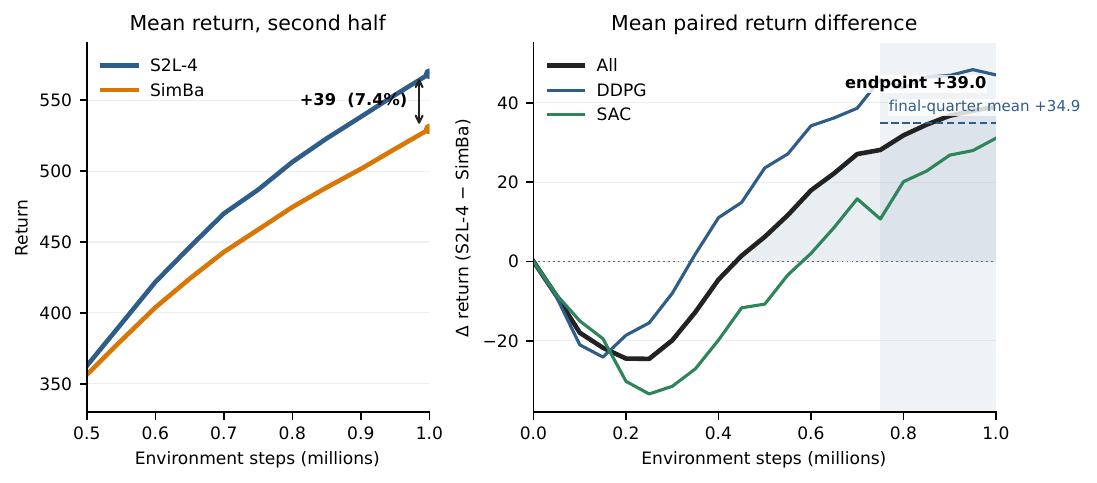}
    \caption{\textbf{The S2L-4 return advantage grows during training.}
    Both panels use the same 720 matched pairs.
    Left: mean return over the second half of training.
    Right: mean paired difference over the full run, overall and by algorithm.
    The shaded window marks the final quarter of training, over which
    the mean paired difference is $+34.99$.}
    \label{fig:dmc-learning}
\end{figure}

% Generated by scripts/make_h38_d_tables.py using reported32 aggregates.
% Historical raw SHA-256 f02d017cf5cf372579246ac9df095ec3c09bb4bd1015a28e4cb6bc1d6b367f73
% Historical late-AUC SHA-256 241e31dcb7c2f8c737ee0bf1a68b3273ee34c64a23c60c58db27d238a702c116
% Reported32 SHA-256 f576fbfeefbea65e1a738f0738afc2ffade564709ae0720a85afe9d3b3ed84b7
\begin{table}[!htbp]
\caption{\textbf{S2L gains over 720 matched pairs.}
Entries combine all 720 pairs. Differences are \hthreed{} minus SimBa; AUCs are divided by
their window length. W counts positive endpoint pairs.}
\label{tab:dmc-main}
\centering
\footnotesize
\setlength{\tabcolsep}{5.0pt}
\begin{tabular}{lrrrrrr}
\toprule
Subset & $n$ & \multicolumn{2}{c}{Final return} & \multicolumn{3}{c}{$\Delta$ AUC}\\
\cmidrule(lr){3-4}\cmidrule(lr){5-7}
& & Mean $\Delta$ & W & 0--1M & 500k--1M & 750k--1M\\
\midrule
All & 720 & 39.09 & 465 & 7.16 & 27.03 & 34.99\\
\midrule
DDPG & 360 & 47.42 & 242 & 17.29 & 40.66 & 46.66\\
SAC & 360 & 30.76 & 223 & $-2.97$ & 13.40 & 23.31\\
\midrule
AdamW & 360 & 47.49 & 252 & 16.84 & 36.61 & 42.79\\
Kron & 360 & 30.68 & 213 & $-2.52$ & 17.45 & 27.18\\
\midrule
dog-run & 180 & 58.65 & 129 & 7.24 & 30.85 & 47.04\\
dog-trot & 180 & 51.26 & 111 & 21.88 & 51.73 & 57.69\\
humanoid-run & 180 & 11.04 & 115 & $-0.10$ & 4.99 & 21.76\\
humanoid-walk & 180 & 35.41 & 110 & $-0.36$ & 20.57 & 67.17\\
\bottomrule
\end{tabular}
\end{table}

% Generated by scripts/make_h38_d_tables.py.
% Stability SHA-256 0f8c5df6a7fe37bdc7891760f2ab91fbdec0cfbc4e1aa8c2717c7adf815d5f86
\begin{table}[!htbp]
\caption{Late-training return dynamics over 720 completed matched pairs.
Statistics use eleven evaluations from 500k to 1M steps. Lower detrended
MAD means less local fluctuation; lower values in the other rows mean
smaller downward return changes. The last column counts pairs with a lower
\hthreed{} value / lower SimBa value / an exact tie.}
\label{tab:dmc-stability}
\centering
\small
\setlength{\tabcolsep}{4.0pt}
\begin{tabular}{lrrr}
\toprule
Metric & \hthreed{} & SimBa & Paired W/L/T\\
\midrule
Local fluctuation (detrended MAD) & \textbf{12.08} & 16.17 & 404/316/0\\
Normalized maximum drawdown & \textbf{0.137} & 0.164 & 426/276/18\\
Cumulative negative change & \textbf{38.26} & 45.32 & 412/290/18\\
Worst 50k-step drop & \textbf{23.85} & 26.62 & 415/287/18\\
\bottomrule
\end{tabular}
\end{table}

\section{Understanding the Design}
\label{sec:matched-mechanisms}

We next examine the three choices that distinguish \shortmethod{}:
transferring features, training the sources with smaller discounts, and
blocking the receiving loss.  We first analyse what smaller discounts do to
targets and what detachment does to feature updates, measuring target drift
directly; Atari controls then test the effect of each choice on return.

\subsection{Target sensitivity and measured drift}
\label{sec:motivation}
\label{sec:target-diagnosis}

We first isolate how the discount affects target construction, then measure
source-target drift during training.  Fix a trajectory from $t$ to truncation
index $T$ and hold rewards fixed.  For $0\leq\gamma<1$ and
$\lambda\in[0,1]$, let $G$ and $\widetilde G$ be the returns in
Eq.~\eqref{eq:qlambda} formed from bootstrap sequences $V$ and
$\widetilde V$.  Writing $\Delta V_j=V_j-\widetilde V_j$ and initially
omitting terminal masks, exact unrolling gives
\begin{equation}
 \Delta G_t
 =(1-\lambda)\sum_{k=1}^{n-1}
   \gamma^k\lambda^{k-1}\Delta V_{t+k}
 +\gamma^n\lambda^{n-1}\Delta V_T,
 \qquad n=T-t.
 \label{eq:qlambda-depth}
\end{equation}
Thus a change $k$ transitions away carries a factor $\gamma^k$; the remaining
factor depends on $\lambda$ but not on the discount.  For the same perturbation,
changing the discount from $\gamma_\ell$ to $\gamma_s<\gamma_\ell$ multiplies
its contribution by $(\gamma_s/\gamma_\ell)^k$.  At $\lambda=1$, only the
terminal term remains, with weight $\gamma^n$.  At $k=32$, changing
$\gamma=.99$ to $\gamma=.9$ reduces the transmitted perturbation by
$21.1\times$.

When bootstrap values at every depth can change, their contributions add.
If $\max_{t<j\leq T}|V_j-\widetilde V_j|\leq\epsilon$, then
\begin{equation}
 |G_t-\widetilde G_t|
 \leq c_n(\gamma,\lambda)\epsilon,\qquad n=T-t,
 \label{eq:qlambda-sensitivity}
\end{equation}
where
\begin{equation}
 c_n(\gamma,\lambda)
 =\gamma(1-\lambda)\frac{1-(\gamma\lambda)^n}{1-\gamma\lambda}
  +(\gamma\lambda)^n.
 \label{eq:qlambda-coefficient}
\end{equation}
The coefficient is tight over arbitrary bootstrap sequences in the sup norm:
equality is attained when all bootstrap changes have the same magnitude and
sign.  Standard terminal masks can only remove terms and reduce it.
Equation~\eqref{eq:qlambda-sensitivity} is the finite-horizon form of standard
$\lambda$-return sensitivity results \citep{kozuno2021peng}.  The bound formalizes the
attenuation of distant bootstrap changes on fixed trajectories.
The proof, including terminal masks, is in Appendix~\ref{app:proofs}.

\paragraph{Measurements during learning.}
We compare mixed-discount and all-$.99$ four-learner models on 16
non-development games and four fresh paired seeds.  At 51 training
checkpoints, both models recompute targets on the same fixed trajectories
from a uniform-random policy and a frozen base policy.
The primary statistic is consecutive target RMS drift divided by each
learner's iteration-1 target standard deviation; we average the three
source learners and both probe sets.  Raw RMS drift is a secondary measure.

Both measures favor mixed discounts in all 64 paired runs and all 16 game
averages (Table~\ref{tab:atari-target-stability}).  The pattern holds for each
source position and both probe sets (Appendix~\ref{app:target-stability}),
linking the target-construction analysis to lower measured source-target
drift during learning.  As a calibration, all 1,671,168 measured
distance-specific coefficients match Eq.~\eqref{eq:qlambda-depth} to within
$1.2\times10^{-8}$.

% Direct target-stability report SHA-256
% dfab13b2a528c2f94aa1fd50a7b724e01d83f1c38402e0537fe4c861e526dfa3
\begin{table}[!htbp]
\caption{\textbf{Fixed-trajectory Q($\lambda$)-target drift.} Entries are
mixed-discount minus all-$.99$ differences in time-averaged consecutive
Q($\lambda$)-target drift, so negative values favor mixed discounts.  The
normalized measure divides each learner and probe stratum by its fixed
iteration-1 target standard deviation.  Intervals use 10,000 whole-game
bootstrap resamples, retaining all four paired seeds.}
\label{tab:atari-target-stability}
\centering
\footnotesize
\setlength{\tabcolsep}{3.4pt}
\begin{tabular}{lrrrr}
\toprule
Target drift & Mean $[95\%\ \mathrm{CI}]$ & Median & Pairs $<0$ & Games $<0$\\
\midrule
Normalized (primary) & $-9.689\;[-17.638,-3.389]$ & $-2.073$ & $64/64$ & $16/16$\\
Raw RMS (secondary) & $-.170\;[-.249,-.101]$ & $-.127$ & $64/64$ & $16/16$\\
\bottomrule
\end{tabular}
\end{table}

\subsection{Gradient isolation controls feature-update gain}

Gradient isolation lets source features follow their own objectives.
By Eq.~\eqref{eq:qlambda-depth}, the longest-horizon target is the most
sensitive to the same bootstrap perturbation; moving value targets are
associated with plasticity loss \citep{sokar2023dormant}.  The following
local shared-feature model separately quantifies how each head's loss
contributes to a feature update and how detachment changes that update.

\begin{theorem}[Discount and shared-feature updates]
\label{thm:target-feature}
Let $\mathcal O$ be a finite set of value heads in the following
two-dimensional prediction model.  Define
$e_1=(1,0)^\top$, $u(\theta)=(\cos\theta,\sin\theta)^\top$, and
$Q_o=a_ou(\theta)$ for scalars $a_o,\theta$ and $o\in\mathcal O$.
The heads share the feature direction $u(\theta)$.  For a scale $R>0$, head
$o$ has discount $\gamma_o$, truncation $(n_o,\lambda_o)$, fixed point
$q_o^\star=R/(1-\gamma_o)$, and coefficient
$c_o=c_{n_o}(\gamma_o,\lambda_o)$.  Its stopped target is
\[
 Y_o=(1-c_o)q_o^\star e_1+c_o\sg[Q_o].
\]
Here $\sg$ is the identity in value and has zero derivative.  With the
Euclidean norm, let
$\ell_o=\tfrac12\|Q_o-Y_o\|_2^2$.  For a nonempty
$\mathcal W\subseteq\mathcal O$, one simultaneous semi-gradient step is
\[
 a_o^+=a_o-\eta\,\partial_{a_o}\ell_o,\qquad
 \theta^+=\theta-\eta\sum_{o\in\mathcal W}\partial_\theta\ell_o,
\]
where all targets and derivatives are evaluated before either update.  Define
\begin{equation}
 K_{\mathcal W}
 =R^2\sum_{o\in\mathcal W}
 \frac{1-c_o}{(1-\gamma_o)^2}.
 \label{eq:target-feature-gain}
\end{equation}
At $\theta^\star=0$ and $a_o^\star=q_o^\star$, the exact local
derivative of the angular update is
\begin{equation}
 \left.\frac{\partial\theta^+}{\partial\theta}\right|_{\theta^\star,a^\star}
 =1-\eta K_{\mathcal W}.
 \label{eq:target-feature-multiplier}
\end{equation}
The linearized angular update is contractive exactly when
$0<\eta K_{\mathcal W}<2$.
Removing a loss from $\mathcal W$ removes its summand from
$K_{\mathcal W}$.  For fixed $n,\lambda$,
\begin{equation}
 (1-c_n)\frac{R^2}{(1-\gamma)^2}
 \sim
 \frac{R^2}{1-\gamma}
 \sum_{j=0}^{n-1}\lambda^j,
 \qquad \gamma\uparrow1.
 \label{eq:target-feature-asymptotic}
\end{equation}
\end{theorem}
The proof is in Appendix~\ref{app:proofs}.

The fixed-point value scale $R^2/(1-\gamma)^2$ grows faster than $1-c_n$
shrinks.  With the Atari parameters $(n,\lambda)=(32,.65)$, the angular
coordinate's linear-stability step-size bound is $11.6\times$ smaller at
$\gamma=.99$ than at $\gamma=.9$.  Detaching the $\gamma=.99$ loss removes
its contribution exactly, allowing source features to retain the update
dynamics set by their own objectives.  The gradient-isolation ablation in
Table~\ref{tab:s2l4-ablations} tests this prediction directly.

\subsection{Four-learner ablations and capacity control}

Table~\ref{tab:s2l4-ablations} reports the original reference scores and
seven controls for \hthreed{}.  Removing feature transfer lowers full-training
AUC by $.0797$; setting all discounts to $.99$ lowers it by $.0652$.
Allowing receiving losses to backpropagate through all three feature
connections lowers full-training and second-half AUC by $.0731$ and $.1530$.
Paired 95\% game-cluster bootstrap intervals on full-training AUC exclude
zero for feature transfer ($+.080\,[+.005,+.165]$) and for mixed discounts
($+.065\,[+.023,+.112]$) over all 57 games, and for detachment in the
matched two-learner comparison ($+.082\,[+.019,+.174]$) over the 52
non-development games; full-training game medians favor detachment on
$42/52$ of them (Appendix~\ref{app:ale57}).

A widened, parameter-matched PQN scores $.1260$ below \hthreed{} at the
endpoint.  Giving all three sources the same short discount, reversing their
horizon order, or freezing them at initialization also lowers all three
metrics.  These comparisons favor trained source features with diverse,
short-to-long horizons and isolated gradients.  Appendix~\ref{app:ale57}
defines the controls and retains the earlier two-learner gradient study.

% Generated by scripts/make_s2l4_ablation_table.py.
% Original reference SHA-256 8163f21b6a83906556f4917a3ff34ab76ca6f692c2ecc7064c0fc11e3db319fd
% Author-supplied aggregate SHA-256 2ad4755fe950744cc81cbf2903ae0e742d179d7a56e7dc2511b908344d883240
\begin{table}[!htbp]
\caption{\textbf{S2L-4 reference scores and ablations.}
Panel (a) gives the original 57-game, four-seed mean HNS scores.
Panel (b) gives changes relative to the corresponding full \hthreed{} model.
Endpoint HNS uses iteration 610; AUCs are divided by their window length.
Intervals for the original paired controls appear in Table~\ref{tab:atari-mechanism-controls}.}
\label{tab:s2l4-ablations}
\centering
\small
\setlength{\tabcolsep}{6pt}
\begin{tabular}{lrrr}
\toprule
\multicolumn{4}{l}{\textit{(a) Original reference scores}}\\
Method & Endpoint HNS & Full AUC & Second-half AUC\\
\midrule
Scaled PQN & $0.9323$ & $0.3161$ & $0.5809$\\
\hthreed{} & $1.0648$ & $0.3999$ & $0.6827$\\
\midrule
\multicolumn{4}{l}{\textit{(b) Variant minus full \hthreed{}}}\\
Variant & $\Delta$ endpoint & $\Delta$ full AUC & $\Delta$ second-half AUC\\
\midrule
No gradient isolation & $-.0971$ & $-.0731$ & $-.1530$\\
No feature transfer & $-.1010$ & $-.0797$ & $-.0914$\\
All discounts $.99$ & $-.0598$ & $-.0652$ & $-.0863$\\
Parameter-matched PQN & $-.1260$ & $-.0791$ & $-.0962$\\
Identical source discounts & $-.1093$ & $-.0689$ & $-.0874$\\
Reversed source horizons & $-.0762$ & $-.0482$ & $-.0593$\\
Frozen source networks & $-.0973$ & $-.0752$ & $-.0744$\\
\bottomrule
\end{tabular}
\end{table}

\subsection{Separate learners versus a shared trunk}
\label{sec:shared-trunk}

\sharedhead{} uses the same discounts but predicts all four values from a
single CNN--MultiSkip trunk, adapting the multi-discount auxiliary-head design
of \citet{fedus2019hyperbolic} to PQN.  Both architectures use the same
interaction budget and four paired seeds on all 57 games.  \sharedhead{} has
approximately 80.4M parameters and \hthreed{} has 321.4M; training details
are in Appendix~\ref{app:ale57}.

The interquartile mean (IQM; \citealp{agarwal2021precipice}) of paired
differences favors \hthreed{} on all four metrics
(Table~\ref{tab:atari-shared-head}); Figure~\ref{fig:atari-shared-head}
shows how this advantage develops during training.

\begin{figure}[t]
    \centering
    \includegraphics[width=\linewidth]{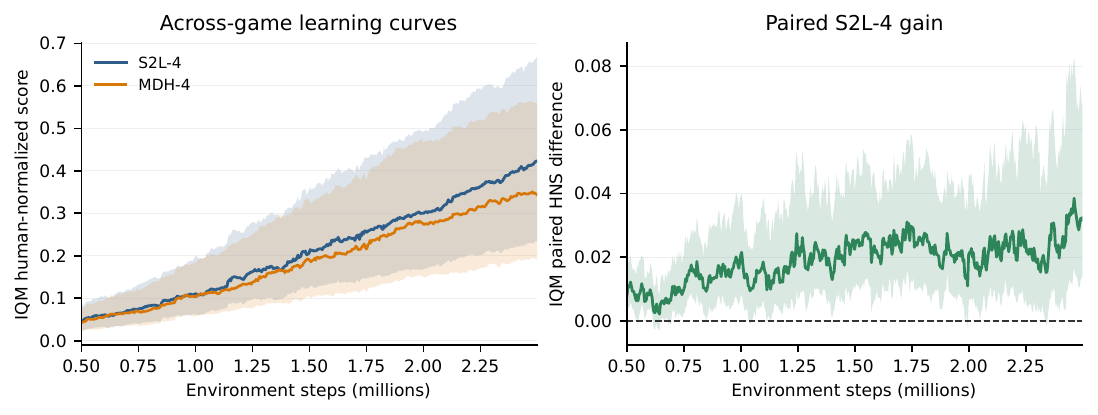}
    \caption{\textbf{S2L-4 versus the shared-trunk MDH-4 baseline.}
    Left: IQM HNS over 228 game--seed runs per method.
    Right: IQM of paired HNS differences, \hthreed{} minus \sharedhead{}.
    Shading gives 95\% game-cluster bootstrap intervals.}
    \label{fig:atari-shared-head}
\end{figure}

A widened \sharedhead{} matches the approximately 321.4M parameters of
\hthreed{} and still scores below it on every summary: mean endpoint HNS
$1.0286$ against $1.0648$, with lower full-training AUC, late-half AUC, and
last-10 HNS (Table~\ref{tab:parameter-matched-mdh4}).  Added capacity alone
therefore does not account for the gain.

\begin{table}[t]
\caption{\textbf{Shared-head control on ALE-57.}  $\Delta$ is \hthreed{}
minus the matched-discount \sharedhead{} baseline in human-normalized
score units; positive values favor \hthreed{}.  Intervals are 95\%
game-cluster bootstrap intervals for the paired-cell IQM (10,000
resamples preserving four seeds per sampled game).  Wins are listed as
seed pairs and game medians.  Terminal HNS has 5 seed-pair ties and 1 game-median tie; Terminal last-10 has 3 seed-pair ties.}
\label{tab:atari-shared-head}
\centering
\small
\begin{tabular}{lrrrr}
\toprule
Summary & Pair wins & Game wins & IQM $\Delta$ & 95\% CI\\
\midrule
\textbf{Terminal HNS} & 147/228 & 35/57 & +.0323 & $[+.0135,\,+.0667]$\\
Full-training AUC & 149/228 & 39/57 & +.0188 & $[+.0089,\,+.0343]$\\
Late-half AUC & 153/228 & 39/57 & +.0250 & $[+.0124,\,+.0489]$\\
Terminal last-10 & 152/228 & 38/57 & +.0329 & $[+.0146,\,+.0726]$\\
\bottomrule
\end{tabular}
\end{table}

\section{Offline Continuous Control}
\label{sec:offline-results}

We evaluate \shortmethod{} on nine Minari MuJoCo datasets
\citep{minari2024}: HalfCheetah, Hopper, and Walker2d, each with simple,
medium, and expert data.  We use TD3+BC
\citep{fujimoto2021minimalist} with SimBa critics, five paired seeds, and
200k updates; the matched control omits feature transfer.

\begin{table}[!htbp]
\caption{\textbf{Offline continuous control.}
Entries are S2L minus its matched TD3+BC control in normalized-score AUC;
parentheses give positive paired seeds.
All nine Minari MuJoCo conditions use five paired seeds and 200k updates.}
\label{tab:offline-rl}
\centering
\small
\setlength{\tabcolsep}{5.0pt}
\begin{tabular}{lrrrr}
\toprule
Environment & Simple & Medium & Expert & Mean \\
\midrule
HalfCheetah & $-1.63$ (0/5) & $+7.32$ (5/5) & $+17.78$ (5/5) & $+7.82$ (10/15) \\
Hopper      & $+3.39$ (5/5) & $+2.61$ (3/5) & $-0.34$ (2/5) & $+1.88$ (10/15) \\
Walker2d    & $-0.39$ (1/5) & $+0.82$ (4/5) & $+1.00$ (3/5) & $+0.47$ (8/15) \\
\midrule
All         & $+0.45$ (6/15) & $+3.58$ (12/15) & $+6.14$ (10/15) & $+3.39$ (28/45) \\
\bottomrule
\end{tabular}
\end{table}

Table~\ref{tab:offline-rl} shows a MuJoCo macro AUC gain of $+3.39$
(hierarchical-bootstrap 95\% interval $[+.17,+7.65]$), with six of nine
dataset means positive.  Appendix~\ref{app:offline-details} records the
offline transfer detail.

\section{Relation to Prior Work}

\paragraph{Multiple horizons.}
\citet{fedus2019hyperbolic} add value heads with multiple discount factors to a
shared Rainbow network, allowing all head losses to update shared features.
TD($\Delta$) represents a long-horizon value through differences between
neighboring discounts \citep{romoff2019separating}, while
\citet{tang2021taylor} derive approximations between values at different
discount factors.  \shortmethod{} instead transfers hidden activations
between separately trained networks, without combining their Q-values.

\paragraph{Representation learning under non-stationarity.}
Auxiliary control tasks \citep{jaderberg2017unreal}, observation or latent
prediction \citep{hafner2020dreamer,schwarzer2021spr}, contrastive learning
\citep{laskin2020curl}, and feature-distribution regularization
\citep{pasand2026isotropic} provide alternative ways to shape representations.
These choices affect representation dynamics
\citep{lyle2021auxiliary,lelan2023bootstrapped,tang2023representation}.
\shortmethod{} learns source features from the environment's own rewards
at shorter horizons and isolates them from the receiving loss.

\paragraph{Deep value optimization.}
PQN enables parallel, replay-free deep TD learning \citep{gallici2024simplifying};
skip connections, LayerNorm, and structured preconditioning improve its
large-network performance \citep{castanyer2025stable}.
SimBa provides a residual architecture for large actor and critic networks
\citep{lee2025simba}.  \shortmethod{} builds on these architectures while
changing which value losses may update transferred features.

\section{Discussion and Conclusion}

\shortmethod{} learns features at multiple horizons while preserving a
single long-horizon behavior objective.  Its fixed four-learner chain
improves online return across Atari and continuous control, with ablations
supporting its component choices.  The analysis links smaller discounts to
attenuated transmission of bootstrap changes into targets and detachment
to isolated feature updates;
fixed-trajectory probes show lower source-target drift during training.
Adaptive horizon selection and more efficient feature sharing are natural
extensions.

\clearpage
\subsection*{AI use statement}

Generative AI tools assisted with literature discovery, experiment
orchestration, source-code implementation and testing, statistical analysis
and interpretation, and manuscript drafting and editing.  Experimental
design and mathematical verification were carried out by the authors.
Figures and tables were produced by the analysis code rather than generated
by these tools.

\subsection*{Ethics statement}

This work studies learning algorithms in simulated Atari and continuous-control
environments and involves neither human subjects nor personal data.  Its main
direct resource costs are computation and energy.  Compared methods use
matched experiment grids.  Where available, valid checkpoints are reused to
reduce repeated computation.

\subsection*{Reproducibility statement}

Appendices~\ref{app:details}, \ref{app:ale57}, and~\ref{app:dmc-details}
record online training and evaluation settings, comparison rules, and
run-inclusion criteria.  Versioned artifacts and analysis
scripts support the empirical summaries.  Training and analysis code will be
released.  Continuous-control
summaries include only matched pairs that reached the stated interaction budget
and produced all required evaluations; exclusions are reported in the
appendix.

\bibliography{references}
\bibliographystyle{iclr2027_conference}

\appendix

\section{Proofs}
\label{app:proofs}

\begin{proof}[Proof of Eq.~\eqref{eq:qlambda-sensitivity}]
Let
\[
 G_i=r_i+\gamma m_{i+1}
 \big[(1-\lambda)V_{i+1}+\lambda G_{i+1}\big],
 \qquad G_T=V_T,
\]
where $m_j\in[0,1]$, and define $\widetilde G_i$ using the same rewards and
masks.  Write
$\Delta V_j=V_j-\widetilde V_j$,
$\Delta G_j=G_j-\widetilde G_j$, and
$M_k=\prod_{\ell=1}^k m_{t+\ell}$.  The common rewards cancel.  Repeated
substitution gives
\begin{equation}
 \Delta G_t
 =(1-\lambda)\sum_{k=1}^{n-1}
   \gamma^k\lambda^{k-1}M_k\Delta V_{t+k}
  +\gamma^n\lambda^{n-1}M_n\Delta V_T.
 \label{eq:qlambda-drift-exact}
\end{equation}
Taking absolute values and using $M_k\leq1$ proves
Eq.~\eqref{eq:qlambda-sensitivity}.  With all masks equal to one,
\[
 c_n(\gamma,\lambda)
 =(1-\lambda)\sum_{k=1}^{n-1}\lambda^{k-1}\gamma^k
  +\lambda^{n-1}\gamma^n.
\]
Choosing every $\Delta V_j$ with the same sign and magnitude attains equality,
so the coefficient is sharp.
\end{proof}

\begin{proof}[Proof of Theorem~\ref{thm:target-feature}]
Let $\delta_o=Q_o-Y_o$.  Numerically,
\[
 \delta_o=(1-c_o)(a_ou(\theta)-q_o^\star e_1),
\]
but $Y_o$ is held constant during differentiation.  With
$u_\perp(\theta)=(-\sin\theta,\cos\theta)^\top$, the gradients are
\begin{align}
 \partial_{a_o}\ell_o
 &=(1-c_o)(a_o-q_o^\star\cos\theta),\\
 \partial_{\theta}\ell_o
 &=(1-c_o)a_oq_o^\star\sin\theta,\qquad o\in\mathcal W.
\end{align}
The first equation applies to every $o\in\mathcal O$; only the second is
summed over $\mathcal W$.  All mixed derivatives between $a_o$ and $\theta$
vanish at the stated fixed point.
After forming these semi-gradients, we differentiate the numerical update
map, recomputing targets at each pre-update point.  This gives head multipliers
$1-\eta(1-c_o)$ and the feature multiplier in
Eq.~\eqref{eq:target-feature-multiplier}, with
$K_{\mathcal W}$ as in Eq.~\eqref{eq:target-feature-gain}.  The angular
multiplier has magnitude below one exactly under the stated condition.
The full linearized joint update is contractive only if, in addition,
$0<\eta(1-c_o)<2$ for every head.

Finally, fix $n\geq1$ and $\lambda\in[0,1]$.  The polynomial form of $c_n$
extends to $\gamma=1$, where $c_n(1,\lambda)=1$, and
\[
 \left.\frac{\partial c_n}{\partial\gamma}\right|_{\gamma=1}
 =\sum_{k=1}^{n-1}k(1-\lambda)\lambda^{k-1}
   +n\lambda^{n-1}
 =\sum_{j=0}^{n-1}\lambda^j.
\]
Therefore
$1-c_n(\gamma,\lambda)\sim
(\sum_{j=0}^{n-1}\lambda^j)(1-\gamma)$.  Substitution proves
Eq.~\eqref{eq:target-feature-asymptotic}; the finite sum equals $n$ at
$\lambda=1$.
\end{proof}

\section{Exact Training Details}
\label{app:details}

\begin{table}[h]
\caption{Base configuration for two-learner Atari experiments.}
\centering
\small
\begin{tabular}{ll}
\toprule
Setting & Value\\
\midrule
Algorithm & PQN with Q($\lambda$)\\
Main / short discount & .99 / .90\\
$\lambda$ & .65\\
CNN & Atari, medium: 32--64--64 channels\\
MLP & width 2560, five two-linear residual blocks\\
Stabilizers & full MultiSkip, LayerNorm, Kron\\
Environments / rollout & 128 / 32 steps\\
Minibatches / epochs & 32 / 2 (64 optimizer calls per iteration)\\
Iterations / environment steps & 610 / 2,498,560\\
Learning rate & $2.5\times10^{-4}/3$\\
Gradient clipping & separate norm 10 per learner\\
Exploration & $1.0\rightarrow0.005$ over first 10\%\\
Seeds & 1, 2, 3, 4\\
\bottomrule
\end{tabular}
\end{table}

For the two-learner variants, the behavior network is constructed first so its
initial parameters and buffers match the baseline exactly; the source network
is constructed second.  Optimizer parameters follow the deterministic order
``all behavior-network parameters, then all source-network parameters.''  Targets are computed
from detached online predictions.  The short target uses its own greedy value,
but the sampled action in both regression losses is the action actually taken
by the behavior learner.

The read-interface study uses Qbert and Phoenix with
$\gamma_s=.9$.  The horizon study uses Qbert, Phoenix, Asterix, Breakout, and
Seaquest and tests effective horizons
$H\in\{4,8,10,16,32,64,100\}$, where $\gamma_s=1-1/H$.  All comparisons use
four paired seeds.  Observations follow ALE-v5 through EnvPool with minimal
actions, frame skip 4, sticky-action probability .25, at most 30 no-ops,
episodic life, and clipped training rewards.

For these development studies, the primary statistic is trapezoidal area
under the recorded training-return curve (AUC).  For a paired candidate $X$ and reference $Y$, the sign-safe
normalized improvement is
\[
 \operatorname{SNI}(X,Y)
 =\frac{2(\operatorname{AUC}_X-\operatorname{AUC}_Y)}
 {|\operatorname{AUC}_X|+|\operatorname{AUC}_Y|},
\]
with value zero when both AUCs are zero.  Per-game source discounts maximize
the four-seed median SNI; the global source maximizes median SNI over all
twenty game--seed pairs, with direction counts and shorter horizons used only
as fixed tie-breakers.  AUC ratios and their geometric means are reported only
when all contributing AUCs are positive.

\section{Feature-Transfer Rules}
\label{sec:banks}
\label{app:read-interfaces}

The two-learner studies use a source $Q^S$ and behavior learner $Q^L$.
Each forms its own target as in Eq.~\eqref{eq:qlambda}, with discounts
$\gamma_s<\gamma_\ell=.99$, and the combined loss is
$\mathcal L_S+\mathcal L_L$.
For long-network block $l$, let $\bar h^L_l$ include the ordinary residual
and skip connections.  The general read interface is
\begin{equation}
 h^L_l=\bar h^L_l+\sg[b^S_l].
 \label{eq:longread}
\end{equation}
Here $s^S_0$ is the source input embedding and $s^S_l$ is its activation
after residual block $l$.

The Atari interface study holds both learners and their updates fixed while
changing only the detached feature $b^S_l$ supplied to long block $l$:
\begin{align}
\textsc{Input:}&& b^S_l &=s^S_0,\\
\textsc{Aligned:}&& b^S_l &=s^S_l,\\
\textsc{Delta:}&& b^S_l &=s^S_l-s^S_{l-1},\\
\textsc{Dense mean:}&& b^S_l &=\frac{1}{6}\sum_{j=0}^{5}s^S_j,\\
\textsc{Prefix mean:}&& b^S_l &=\frac{1}{l+1}\sum_{j=0}^{l}s^S_j.
\end{align}
No interface adds parameters, optimizer state, target mixing, or an additional
behavior policy.

\section{Atari Feature-Transfer Results}
\label{app:atari-read}

\begin{table}[h]
\caption{Paired training-return AUC relative to the scaled PQN base agent of
\citet{castanyer2025stable}.  GM is the geometric mean across the listed
paired seeds.  Every method uses the same $\gamma=.99$ target construction
and action-selection rule.}
\label{tab:main}
\centering
\small
\begin{tabular}{lcccccc}
\toprule
& \multicolumn{2}{c}{Qbert (4 seeds)} & \multicolumn{2}{c}{Phoenix (4 seeds)}
& \multicolumn{2}{c}{Pooled (8)}\\
\cmidrule(lr){2-3}\cmidrule(lr){4-5}\cmidrule(lr){6-7}
Transferred feature & GM & Wins & GM & Wins & GM & Wins\\
\midrule
None (baseline) & 1.000 & -- & 1.000 & -- & 1.000 & --\\
Short input & 1.173 & 4/4 & 1.173 & 4/4 & 1.173 & 8/8\\
Aligned absolute & 1.239 & 4/4 & 1.237 & 4/4 & 1.238 & 8/8\\
Aligned delta & 1.188 & 4/4 & 1.199 & 4/4 & 1.194 & 8/8\\
Dense all-layer mean & \textbf{1.277} & 4/4 & 1.240 & 4/4 & \textbf{1.258} & 8/8\\
Prefix mean & 1.222 & 4/4 & \textbf{1.264} & 4/4 & 1.243 & 8/8\\
\bottomrule
\end{tabular}
\end{table}

\begin{figure}[h]
    \centering
    \includegraphics[width=\linewidth]{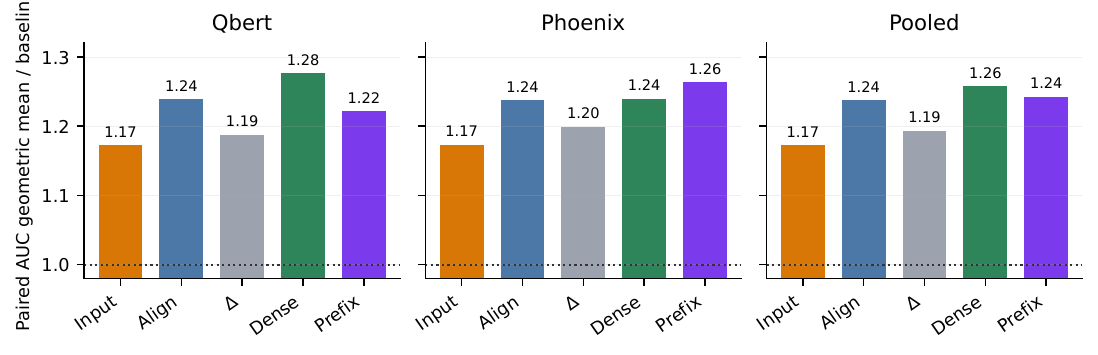}
    \caption{\textbf{Feature-transfer comparison.}  Bars show paired
    geometric-mean training-return AUC relative to the $\gamma=.99$ base
    agent on Qbert and Phoenix.  Every two-learner variant has the same
    parameter count, target construction, optimizer configuration, and action-selection rule; only the
    transferred feature changes.}
    \label{fig:read-ablation}
\end{figure}

\begin{table}[h]
\caption{Feature-transfer effects relative to the detached short-input
variant.  Parameter counts, optimizer settings, target construction, and
action-selection rules match.}
\label{tab:read}
\centering
\small
\begin{tabular}{lcccc}
\toprule
Transferred feature & Qbert GM & Phoenix GM & Pooled GM & Positive pairs\\
\midrule
Short input & 1.000 & 1.000 & 1.000 & --\\
Aligned absolute & 1.0565 & 1.0546 & 1.0556 & 6/8\\
Aligned delta & 1.0129 & 1.0221 & 1.0175 & 4/8\\
Dense all-layer mean & \textbf{1.0887} & 1.0566 & \textbf{1.0725} & 6/8\\
Prefix mean & 1.0420 & \textbf{1.0773} & 1.0595 & \textbf{7/8}\\
\bottomrule
\end{tabular}
\end{table}

\section{Five-Game Discount Selections}
\label{app:discount}

\begin{table}[h]
\caption{Per-game source-discount-factor choices from the five-game sweep.  Median SNI
compares the selected source with the scaled PQN parent of
\citet{castanyer2025stable} over four paired seeds.}
\label{tab:discount}
\centering
\small
\begin{tabular}{lccc}
\toprule
Game & Selected $\gamma_s$ & Effective horizon & Median SNI\\
\midrule
Qbert & .875 & 8 & .256\\
Phoenix & .900 & 10 & .154\\
Asterix & .875 & 8 & .106\\
Breakout & .984375 & 64 & .064\\
Seaquest & .96875 & 32 & .341\\
\bottomrule
\end{tabular}
\end{table}

\begin{figure}[h]
    \centering
    \includegraphics[width=\linewidth]{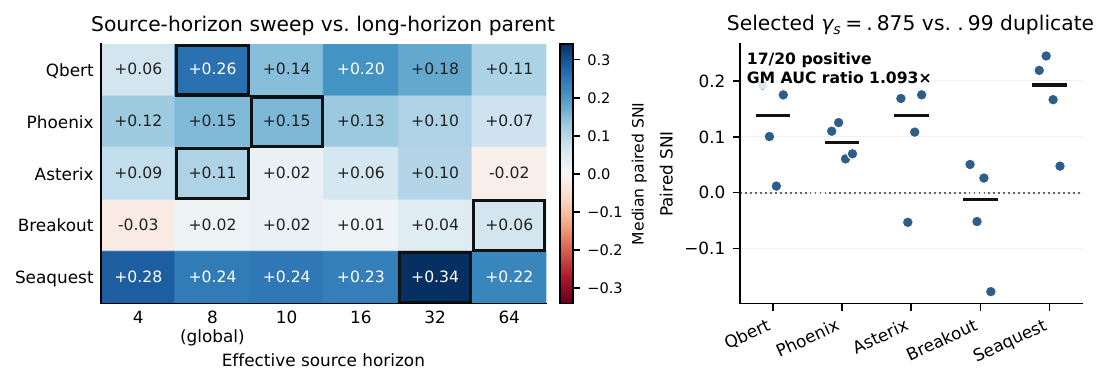}
    \caption{\textbf{Source-discount selection.}  Left: four-seed
    median paired SNI against the $\gamma=.99$ base agent across the eligible grid;
    outlined cells are in-sample per-game maxima and $H=8$ is selected globally
    on the same panel.  Right: seed-level paired SNI for that selected source
    against the matched $\gamma_s=.99$ duplicate source.  SNI is defined in
    Appendix~\ref{app:details}.}
    \label{fig:atari-horizon}
\end{figure}

\section{Full ALE-57 Evaluation}
\label{app:ale57}

The fixed four-learner model uses the discounts in
Eq.~\eqref{eq:h38d-gammas} and the aligned predecessor connection in
Eq.~\eqref{eq:h38d-chain}.  It is compared with the scaled PQN parent of
\citet{castanyer2025stable} on all 57 ALE games, four paired seeds, and
2,498,560 environment steps per run.
For each game--seed pair, terminal HNS is the mean of learner iterations
601--610 of the logged rolling-20 HNS.  Figure~\ref{fig:atari57-gain} reports
paired differences directly in HNS units, accommodating signed and near-zero
scores.  Common-support full-training AUC is a secondary curve-level summary.

\paragraph{Original paired controls.}
The full-suite control matrix uses the same ALE protocol, optimizer,
interaction budget, and four paired seeds.  The four-discount aligned detached
model is compared with a no-transfer control and a control that keeps the
connections but sets all four discounts to $.99$.  These three arms contain
$57\times4$ runs each.  The two-learner detached and writable controls contain
$52\times4$ runs each, excluding the five development games.  Both use the
fixed $\gamma_s=.875$ mean-pooled design; they differ only in whether the
$\gamma=.99$ loss is stopped at the source feature or backpropagates through
it.  Initialization, forward values before an update, targets, loss weights,
per-learner clipping, and the single Kron step are matched.

All 1,100 curves reach iteration 610 and enter the analysis.  Endpoint HNS is
the rolling-20 value at iteration 610; AUCs are normalized trapezoidal
averages over the pair-common finite support, with the late window beginning at
iteration 305.  Reported intervals use 10,000 whole-game bootstrap resamples
and preserve all four paired seeds.

% Generated by scripts/make_h41_mechanism_table.py.
% Terminal report SHA-256 f22ccf51bff1bf5886e885b27a50d1800d2f9f312ec03279f9ea6bf66187d011
\begin{table}[H]
\caption{\textbf{Matched Atari ablations.} Entries are mean paired
human-normalized-score differences. Brackets give 95\% intervals from
10,000 bootstrap resamples of games, retaining all four paired seeds.
The first two rows cover all 57 games; the last covers the 52 games not
used for development. Positive values favor the first method.}
\label{tab:atari-mechanism-controls}
\centering
\footnotesize
\setlength{\tabcolsep}{3.3pt}
\begin{tabular}{lrrr}
\toprule
Contrast & Endpoint HNS & Full AUC & Second-half AUC\\
\midrule
Transfer $-$ no transfer & $+.101\;[-.121,+.282]$ & $+.080\;[+.005,+.165]$ & $+.091\;[-.024,+.207]$\\
Mixed discounts $-$ all $.99$ & $+.060\;[-.101,+.202]$ & $+.065\;[+.023,+.112]$ & $+.086\;[+.028,+.153]$\\
Detached $-$ backpropagated & $+.114\;[+.006,+.256]$ & $+.082\;[+.019,+.174]$ & $+.122\;[+.025,+.275]$\\
\bottomrule
\end{tabular}
\end{table}

In the two-learner comparison, detachment improves endpoint, full-training
AUC, and second-half AUC (Table~\ref{tab:atari-mechanism-controls}).
Full-training game medians favor detachment on $42/52$ games, with mean and
median AUC gains of $+.0817$ and $+.0082$.

\paragraph{Extended S2L-4 controls.}
The four-network gradient ablation removes stop-gradient from all three
predecessor connections while retaining the forward computation.
The identical-source-discount variant uses $(.75,.75,.75,.99)$; the
reversed-source variant uses $(34/35,11/12,.75,.99)$.
The frozen-source variant retains detached additions but holds all three
source networks fixed at random initialization and trains only the final
network.  The parameter-matched PQN control widens a single network to match
the total capacity of \hthreed{}.  These additional controls reuse the
existing full-model reference scores and are summarized as aggregate
differences in Table~\ref{tab:s2l4-ablations}.

\paragraph{Shared-trunk comparison.}
The \sharedhead{} baseline uses one CNN--MultiSkip trunk with four Q heads at
the same discounts as \hthreed{}, $\lambda=.65$, and behavior controlled only
by the $.99$ head.  Its objective is the $.99$-head MSE plus the mean of the
three auxiliary-head MSEs.  All 57 games, four paired seeds, ALE protocol
fields, and 2,498,560-step budgets match the \hthreed{} comparison.  The primary
summary is HNS at learner iteration 610; full-training AUC, late-half AUC, and
the mean of the final ten logged returns are secondary.  The paired-cell IQM
trims 25\% from each tail, and its interval resamples whole games while
preserving all four seeds.  All summaries use signed HNS differences.
Figure~\ref{fig:atari-shared-head} starts at 0.5M steps; all 456 arm curves
have valid returns by step 425,984.

The shared-trunk implementation uses global gradient clipping and Kron
preconditioning.  \hthreed{} uses per-learner clipping and preconditioning,
a sum of the four losses, and detached feature transfer.  These settings
define the complete architectures compared in Section~\ref{sec:shared-trunk}.

\paragraph{Parameter-matched shared trunk.}
A further control widens the \sharedhead{} trunk to match the total parameter
count of \hthreed{}. All three models use the same 57 games, seeds 1--4,
and 2,498,560-step budget. Table~\ref{tab:parameter-matched-mdh4} gives
arithmetic means and standard deviations across the 228 game--seed runs.

Endpoint HNS uses iteration 610, and last-10 averages iterations 601--610
within each run. AUCs use normalized trapezoidal integration, with the late
window covering iterations 305--610. The \hthreed{} full-training AUC
reference uses common valid support with \sharedhead{}, so its mean differs
slightly from the PQN-paired reference in Table~\ref{tab:s2l4-ablations}.

% Generated by scripts/make_parameter_matched_mdh4_table.py.
\begin{table}[!htbp]
\caption{\textbf{Parameter-matched shared-trunk control on ALE-57.}
Entries are arithmetic means $\pm$ standard deviations across 57 games
and four seeds, in HNS units. Wide \sharedhead{} matches the total
parameter count of \hthreed{}.
Standard deviations describe score dispersion across game--seed cells.}
\label{tab:parameter-matched-mdh4}
\centering
\small
\setlength{\tabcolsep}{5pt}
\begin{tabular}{lrrr}
\toprule
Summary & \hthreed{} & \sharedhead{} & \sharedhead{} (wide)\\
Parameters (approx.) & 321.4M & 80.4M & 321.4M\\
\midrule
Endpoint HNS & $1.0648\,\pm\,2.2205$ & $1.0181\,\pm\,2.7107$ & $1.0286\,\pm\,2.7339$\\
Full-training AUC & $0.3997\,\pm\,0.9289$ & $0.3627\,\pm\,0.8405$ & $0.3511\,\pm\,0.9083$\\
Late-half AUC & $0.6827\,\pm\,1.3488$ & $0.6400\,\pm\,1.4027$ & $0.6174\,\pm\,1.4369$\\
Terminal last-10 & $1.0278\,\pm\,2.1008$ & $0.9896\,\pm\,2.5857$ & $0.9319\,\pm\,2.6126$\\
\bottomrule
\end{tabular}
\end{table}

\section{Fixed-Trajectory Target-Drift Protocol}
\label{app:target-stability}

The direct measurement uses 16 games selected before observing these
diagnostics by hashing the canonical names of the 52 non-development games.
For each game, four fresh paired seeds compare the mixed-discount model with
the otherwise matched all-$.99$ control.  Every learner is evaluated in situ
at 51 fixed iterations from 1 through 610; diagnostic evaluation neither
advances an environment nor changes training RNG, parameters, or optimizer
state.

Each game has two immutable probe strata shared byte-for-byte across methods,
seeds, and landmarks: 256 uniform-random and 256 frozen base-policy greedy
trajectories, each containing 32 transitions.  Observations, actions, clipped
rewards, terminal masks, and episode boundaries are therefore held fixed.  At
each landmark we recompute the implemented Q($\lambda$) target and greedy
bootstrap values.  For source learner $o$, probe stratum $p$, and landmark
$k>0$, the normalized consecutive drift is
\[
 d_{o,p,k}=
 \frac{\operatorname{RMS}_i\!\left[G_{o,p,k}(i)-G_{o,p,k-1}(i)\right]}
 {\max\{\operatorname{sd}_i[G_{o,p,0}(i)],10^{-6}\}}.
\]
We integrate $d_{o,p,k}$ by the trapezoidal rule over actual environment-step
coordinates and divide by the elapsed span.  The primary cell statistic
averages the first three learner positions with weight $1/3$ and the two probe
strata with weight $1/2$; the final $.99$ learner is a diagnostic control.
Raw RMS drift is aggregated identically.  Paired differences are formed
within game and seed, and 95\% intervals use 10,000 whole-game bootstrap
resamples that retain all four paired seeds and the complete time series.

Both measures favor mixed discounts in every paired run and every game
(Table~\ref{tab:atari-target-stability}).  The pattern also holds separately
for every primary source position and both probe strata.  As a calibration of
Eq.~\eqref{eq:qlambda-depth}, 1,671,168 measured distance-specific
coefficients have zero mismatches at tolerance $10^{-7}$; the largest absolute
error is $1.2\times10^{-8}$.  Across 417,792 within-arm monotonicity checks,
there are no violations.

\section{Continuous-Control Training Details}
\label{app:dmc-details}

\begin{table}[h]
\caption{Shared settings for the \hthreed--SimBa DMC matrix.}
\centering
\small
\begin{tabular}{ll}
\toprule
Setting & Value\\
\midrule
Tasks & dog-run, dog-trot, humanoid-run, humanoid-walk\\
Algorithms / optimizers & SAC, DDPG / AdamW, Kron\\
Environment steps / action repeat & 1,000,000 / 2\\
Train environments / updates per interaction & 1 / 2\\
Replay capacity / warm-up / batch & 1M / 5,000 / 256\\
Evaluation & 10 episodes every 50,000 environment steps\\
SimBa discount / $n$-step / target $\tau$ & .99 / 1 / .005\\
\hthreed{} discounts & $.75,\ 11/12,\ 34/35,\ .99$\\
AdamW actor/critic learning rate & $10^{-4}$\\
Kron actor/critic learning rate & $5\times10^{-5}$\\
Actor/critic weight decay & $10^{-2}$\\
Seeds & 0, 1000, 2000, 3000, 4000\\
\bottomrule
\end{tabular}
\end{table}

The nine scale points are: base $(1,128;2,512)$; depth-1/2/3
$(2/3/4,128;4/6/8,512)$; width-1/2/3
$(1,256/384/512;2,1024/1536/2048)$; and joint-1/2
$(2,256;4,1024)$ and $(3,384;6,1536)$.  Each tuple lists
$(\text{actor blocks},\text{actor width};\text{critic blocks},\text{critic
width})$.  The base point shared by the depth, width, and joint sweeps appears
once, yielding nine rather than eleven unique scales.

For every matched cell, the final \hthreed{} critic is constructed under the
same module names and before the three auxiliary critics, preserving the
baseline final-critic initialization exactly.  TD updates stop gradients
through predecessor features and sum the four per-critic MSE losses without
averaging, so each critic retains the baseline loss scale.  SAC entropy
terms enter every discount-specific target.  DDPG uses the same exploration-noise
distribution as its SimBa cell.  The actor uses only the final $.99$ critic
value, with predecessor action Jacobians live and critic parameters held fixed.

\paragraph{Evaluation and support.}
The full design contains 720 matched pairs:
two algorithms, two optimizers, four tasks, five seeds, and nine scales.
Every run trains for one million environment steps and is evaluated for ten
episodes every 50,000 steps.  Table~\ref{tab:dmc-main} combines all 720 paired
evaluations.  Means are weighted by pair count and
wins are added within each subgroup.

Performance is the paired difference \hthreed{} minus SimBa in endpoint return
or trapezoidal AUC divided by its window length.  Full AUC covers $0$--$1$M
steps; late AUC covers $500$k--$1$M or $750$k--$1$M.  Return-dynamics diagnostics use
the eleven evaluations from $500$k to $1$M.  Detrended residual MAD measures
local deviation around a linear trend; range-normalized maximum drawdown and
cumulative negative variation measure the loss of previously attained return.

\paragraph{Seed variability.}
Across 144 task--algorithm--optimizer--scale configurations with five seeds
each, the mean cross-seed endpoint MAD is 31.79 for \hthreed{} and 35.21 for
SimBa; \hthreed{} is lower in 88 configurations, SimBa in 50, with six ties.  For
$500$k--$1$M AUC, the corresponding mean cross-seed MADs are 35.97 and 46.21,
with \hthreed{} lower in 87 configurations.

\section{Offline-Control Implementation Detail}
\label{app:offline-details}

For the offline runs, detached source activations are RMS-capped at the
recipient activation RMS (cap $=1$) before transfer.

\section{Seed-Level Terminal AUC}
\label{app:seeds}

Table~\ref{tab:seedauc} gives the training-return AUC values for the
two-game development study in Appendix~\ref{app:atari-read}.
Values are divided by $10^9$ only for presentation.

\begin{table}[h]
\caption{Training-return AUC ($10^9$) for baseline (A), simple short input (B),
dense mean (E), and prefix mean (F).}
\label{tab:seedauc}
\centering
\small
\begin{tabular}{llrrrr}
\toprule
Game & Seed & A & B & E & F\\
\midrule
Qbert & 1 & 4.194 & 4.888 & 5.754 & 5.532\\
      & 2 & 3.973 & 5.363 & 5.330 & 4.760\\
      & 3 & 4.722 & 4.950 & 5.468 & 5.138\\
      & 4 & 4.365 & 5.003 & 5.441 & 5.658\\
\midrule
Phoenix & 1 & 6.956 & 8.190 & 8.086 & 8.430\\
        & 2 & 6.972 & 8.464 & 8.619 & 8.576\\
        & 3 & 6.231 & 7.204 & 8.448 & 8.520\\
        & 4 & 6.745 & 7.737 & 8.178 & 8.450\\
\bottomrule
\end{tabular}
\end{table}

\section{Two-Learner Atari Update}

\begin{enumerate}
    \item Collect a rollout with epsilon-greedy actions from $Q^L$ only.  Store
    long and short greedy values for their respective Q($\lambda$) recursions.
    \item Compute $G^L$ with $\gamma_\ell=.99$ and $G^S$ with $\gamma_s$ on the
    same rewards, masks, transitions, and minibatch permutation.
    \item For each minibatch, run the source learner to obtain
    $Q^S$ and activations $s^S_0,\ldots,s^S_5$.
    \item Form the selected feature $b^S_l$, detach it, and add it after the
    ordinary long MultiSkip addition at long block $l$.
    \item Compute $\mathcal L_L+\mathcal L_S$, backpropagate once, clip the
    behavior and source learners separately, and take one Kron step.
\end{enumerate}

\end{document}